\documentclass{article}

\usepackage[margin=1.25in]{geometry}
\usepackage[utf8]{inputenc} 
\usepackage[T1]{fontenc}    
\usepackage{hyperref}       
\usepackage{url}            
\usepackage{booktabs}       
\usepackage{amsfonts}       
\usepackage{nicefrac}       
\usepackage{microtype}      %
\usepackage[toc,page]{appendix}

\usepackage{graphicx}
\usepackage{subfigure}
\usepackage{booktabs}
\usepackage{hyperref}
\usepackage{amsmath}
\usepackage{amssymb}
\usepackage[dvipsnames]{xcolor}
\usepackage{cleveref}
\usepackage{algorithm}
\usepackage[noend]{algpseudocode}
\usepackage{amsthm}
\usepackage{natbib}
\usepackage{authblk}

\newtheorem{theorem}{Theorem}

\begin{document}

\title{A general method for regularizing tensor decomposition methods via pseudo-data}
\author[1]{\textbf{Omer Gottesman}\thanks{\href{mailto:gottesman@fas.harvard.edu}{gottesman@fas.harvard.edu}}}
\author[1]{\textbf{Weiwei Pan}}
\author[1]{\textbf{Finale Doshi-Velez}\thanks{\href{mailto:finale@seas.harvard.edu}{finale@seas.harvard.edu}}}
\affil[1]{Paulson School of Engineering and Applied Sciences, Harvard University}
\date{}

\maketitle

\begin{abstract}

Tensor decomposition methods allow us to learn the parameters of latent variable models through decomposition of low-order moments of data. A significant limitation of these algorithms is that there exists no general method to regularize them, and in the past regularization has mostly been performed using bespoke modifications to the algorithms, tailored for the particular form of the desired regularizer. We present a general method of regularizing tensor decomposition methods which can be used for any likelihood model that is learnable using tensor decomposition methods and any differentiable regularization function by supplementing the training data with pseudo-data. The pseudo-data is optimized to balance two terms: being as close as possible to the true data and enforcing the desired regularization.  On synthetic, semi-synthetic and real data, we demonstrate that our method can improve inference accuracy and regularize for a broad range of goals including transfer learning, sparsity, interpretability, and orthogonality of the learned parameters.

\end{abstract}

\section{Introduction}
\label{sec:intro}

Tensor decomposition methods (TDMs) have recently gained popularity as ways of performing inference for latent variable models~\citep{anandkumar2014tensor}. The interest in these methods is motivated by the fact that they come with theoretical global convergence guarantees in the limit of infinite data~\citep{anandkumar2012spectral,arora2013practical}. However, a main limitation of these methods is that they lack natural methods for regularization or encouraging desired properties on the model parameters when the amount of data is limited.

Previous works attempted to alleviate this drawback by modifying existing tensor decomposition methods to incorporate specific constraints, such as sparsity~\citep{sun2015provable}, or incorporate modeling assumptions, such as the existence of anchor words~\citep{arora2013practical,nguyen2014anchors}.  All of these works develop bespoke algorithms tailored to those constraints or assumptions.  Furthermore, many of these methods impose hard constraints on the learned model, which may be detrimental as the size of the data grow---framed in the context of Bayesian intuition, when we have a lot of data, we want our methods to allow the evidence to overwhelm our priors.

We introduce an alternative approach which can be applied to encourage \emph{any} (differentiable) desired structure or properties on the model parameters, and which will only encourage this ``prior'' information when the data is insufficient.  Specifically, we adopt the common view of Bayesian priors as representing ``pseudo-observations'' of artificial data which bias our learned model parameters towards our prior belief~\citep{bishop2006pattern}. We apply the tensor decomposition method of~\citet{anandkumar2014tensor} to data sets comprised of both the actual data and an artificial pseudo-data.  Gradient descent and automatic differentiation~\citep{baydin2015automatic,maclaurin2015autograd} can then be used to optimize our pseudo-data such that they maximize the desired properties on the inferred model parameters while still remaining as similar as possible to the actual data.

The resulting algorithm provides a method of imposing any regularizer on the standard TDMs and can be applied to any likelihood model which is learnable using TDMs. We provide theoretical analysis and prove that in the limit of infinite training data, the results of our algorithm converge to the results of standard TDMs which are known to be consistent. We empirically demonstrate our method can regularize for a wide range of properties---knowledge transfer, sparsity, interpretability, and orthogonality of the learned parameters---on two likelihood models---Gaussian mixtures and Latent Dirichlet Allocation (LDA).

\section{Related Work}
\label{sec:related_work}
There exists a very large literature on algorithms for performing inference on latent variable models with various kinds of priors or regularizers.  When placing this work in the context of previous and related work, we emphasize that while we draw inspiration from Bayesian statistics, where priors can be regarded as representing pseudo-observations, our formulation is \emph{not} Bayesian---we do not translate between the pseudo-data and an actual probability distribution on the model parameters. Rather, the pseudo-data should be thought of as a regularization scheme particularly well-suited for tensor decomposition approaches.

Several works have adapted standard tensor decomposition methods to incorporate specific regularizers and constraints: \citet{sun2015provable} develop methods designed to find sparse decompositions; \citet{nguyen2014anchors} produce more robust inference for topic models via the notion of anchor words~\citep{arora2013practical}. (While not strictly a regularization, the anchor words assumption is a form of prior knowledge imposed on the model parameters.) \citet{cohen2013experiments,duchi2008efficient} detail how to ensure that learned parameters do not have invalid values (e.g. topic models must be valid probability distributions).  All of these approaches are specific to the kind of regularization or constraint, and often specific to a certain generative model (e.g. LDA). 

Finally, closely related to tensor decomposition methods are methods based on the generalized method of moments (GMM)~\citep{hansen1982large}: both learn the structure of a distribution from low order moments of the data. Recently, several methods have been proposed to regularize GMMs \citep{tran2016spectral, lewis2018adversarial, yin2009bayesian}. In general, however, most GMM algorithms do not leverage the moments eigenstructure to perform the provably optimal inference which is one of the main draws of TDMs.

\section{Background and Notation}
\label{sec:background}

Many common generative models have parameters that can be expressed as a matrix $A\in\mathbb{R}^{D\times K}$, where $D$ is the dimensionality of the data and $K$ is the number of latent variables.  For example, the columns of $A$ could represent the means of Gaussian mixtures or the topic-word probabilities in LDA~\citep{blei2003latent}.  Tensor decomposition methods leverage the relationship between the empirical moments of the data and the latent parameters of the model. Specifically, $A$ is learned by matching the low order moments of the model parameters,
\begin{align}
M_2 & = \sum_{k=1}^K \beta_k a_k a_k^T, \nonumber \\
M_3 & = \sum_{k=1}^K \gamma_k a_k \otimes a_k \otimes a_k, \label{eq:decomposition}
\end{align}
with empirical estimates, $\hat{M}_2$ and $\hat{M}_3$, that can be computed from data using expressions which are specific to the particular likelihood model and are presented in Appendix \ref{sec:empirical_moments}. Here, $a_k$ is the $k^{th}$ column of $A$, and $\beta_k$ and $\gamma_k$ are constants that come from the model parameters (depend on each generative model). Throughout this paper, we use $a_k$ for the column vectors of $A$ and $a_{d,k}$ for individual elements of $A$.

The decomposition to solve for $\{a_k\}_{k=1}^K$ in Equation~\ref{eq:decomposition} given estimates of $M_2$ and $M_3$ is performed in two stages.  We outline the idea here and refer the reader to \citet{anandkumar2014tensor} for details. First, one computes a whitening matrix $W$ such that $W^T \hat{M}_2 W=I$, and use $W$ to project $\hat{M}_3$ to a $\mathbb{R}^{K\times K \times K}$ tensor which has an orthogonal decomposition. Then, the tensor power method can be used to decompose the reduced third order tensor.  This process comes with theoretical guarantees on consistency and convergence for recovering $A$ with sufficiently large data sets.

\section{Pseudo-data for Regularization}
\label{sec:pseudo_data_as_priors}
While tensor decomposition methods come with consistency and convergence guarantees, it is not obvious how to incorporate regularization appropriately.  The core issue is that tensor decomposition methods provide an algorithm that produces a point estimate of the latent parameters $A$.  There exists no log-likelihood and log-prior, as one would encounter in MAP estimation, nor is there an objective where one can simply add an arbitrary regularizer with some desired strength.  If we modify the latent parameters, $A$, that are output from a tensor decomposition method, we have no notion of how much predictive quality we sacrifice.

To solve this problem we draw on the intuition which is often used when describing Bayesian priors, of viewing the priors as encoding pseudo-data which match our expectation of the data. Given a prior or a regularizer, we can choose our pseudo-data in a way which will drive the inferred latent parameters towards a form which is more in line with our expectations.  We balance that goal with the requirement that our pseudo-data must also be likely under the model parameters we learn using only the real data, which ensures that the pseudo-data are not too different from the real data.

\paragraph{Objective}
Formally, let $X_T \in \mathbb{R}^{D \times N_T}$ be a set of $N_T$ observations of true data.  Let $X_P \in \mathbb{R}^{D \times N_P}$ be a set of $N_P$ observations of artificial pseudo-data.  We seek a pseudo-dataset $X_P$ that maximizes the following objective
\begin{align}
\label{eq:cost_func__general}
L(X_T,X_P,\lambda) = -\log p(X_P|A_T) + \lambda R(A_{T\cup P}),
\end{align}
where $A_T$ and $A_{T\cup P}$ are the parameter matrices learned by the TDM using either only the real training data, $X_T$, or a combination of both the real and pseudo-data, $X_{T\cup P}$, respectively. The first term in Equation \ref{eq:cost_func__general} is the conditional probability of the pseudo-data, $X_P$, given the generative model for the data (ex. LDA or Gaussian mixtures) with the parameters $A$ learned using only the training data.  This term encourages the pseudo-data, $X_P$, to be similar to the true data, $X_T$. The regularizer, $R(A_{T\cup P})$, could be any regularizer encoding our desired characteristics for the parameters $A$.  The weight $\lambda$ controls the relative strengths of the regularizer and the likelihood of the pseudo-data.

\begin{algorithm}[tb]
\begin{algorithmic}[1]
   \caption{Regularized Tensor Decomposition Method (RTDM)}
   \label{alg:RTDM}
   \State {\bfseries Input:} input data~-~$X_T$, regularizing function~-~$R(A_{T\cup P})$, number~of~pseudo-data~-~$N_P$, regularization constant~-~$\lambda$, convergence criterion~-~$\epsilon$
   \State $A_T\leftarrow \text{TDM}(X_T)$
   \State $X_P' \leftarrow \infty$
   \State Draw $X_P$ from $p(X_P|A_T)$
   \While{$||X_P-X_P'||_2 > \epsilon$}
   \State $X_P' \leftarrow X_P$
   \State $X_P \leftarrow X_P - \text{ADAM}(\nabla_{X_P}L(X_T,X_P,\lambda))$
   \EndWhile
   \State $A_{T\cup P} \leftarrow \text{TDM}(X_{T\cup P})$
   \State {\bfseries Return:} $A_{T\cup P}$
\end{algorithmic}
\end{algorithm}

In addition to the weight $\lambda$, the objective in Equation~\ref{eq:cost_func__general} requires us to choose the number of pseudo-data points, $N_P$. The number of pseudo-data points $N_P$ represents how much confidence we put in our prior knowledge of the parameters structure. An advantage of specifying the strength of the regularizer via the number of pseudo-data is that it limits how much the pseudo-data can influence the learned parameters: as the number of training samples $N_T$ increases, for any finite $\lambda$, the maximum effect of the pseudo-data on the inferred parameters diminishes. The dominance of the training data $X_T$ as $N_T \gg N_P$ represents the tendency to put less weight on the pseudo-data as more real data is collected. This notion of convergence is formalized in Theorem \ref{theorem:convergence}.

\begin{theorem}
\label{theorem:convergence}
Fix a likelihood model, $p(X|A)$, and a regularizer, $R(A)$, that is bounded below by a constant $B_R$. For any fixed nonnegative $\lambda$ and $N_P$, as $N_T \rightarrow \infty$, minimizing the loss function, $L$, in Equation \ref{eq:cost_func__general} with respect to $X_P$ results in $A_{T \cup P} \rightarrow A_T$.
\end{theorem}

The proof of Theorem \ref{theorem:convergence} is presented in Appendix \ref{sec:proof_of_consistency}. Furthermore, in Appendix \ref{sec:proof_of_consistency} we also show that if we use plug-in estimators for the bounds of the standard TDMs (e.g. like those in \citet{anandkumar2012spectral, hsu2013learning}), generally the convergence to the true parameters is of the form $A_{T \cup P} = A + \mathcal{O}(\frac{N_P}{N_T})$ as a consequence of Theorem \ref{theorem:convergence}.

\paragraph{Optimization Procedure}
Algorithm \ref{alg:RTDM} describes the learning procedure for the regularized tensor decomposition method (RTDM). The function TDM() denotes the model parameters learned using the tensor decomposition algorithm, and ADAM() is the gradient descent step based on the ADAM algorithm~\citep{kingma2014adam}. 

Computing the gradient of the loss $L(X_T,X_P,\lambda)$ requires taking the gradient of the two terms in Equation~\ref{eq:cost_func__general}.  The first depends directly on the pseudo-data $X_P$; most likelihoods are straight-forward to compute via standard auto-differentiation packages, e.g. the Python Autograd package~\citep{maclaurin2015autograd}. The second term depends on $X_P$ implicitly via $A_{T\cup P} = \text{TDM}(X_{T\cup P})$.  Fortunately, the standard TDM algorithm consists entirely of linear algebra computations which also renders the second term amendable to automatic differentiation. We initialize our pseudo-data by drawing samples from $p(X_P|A_T)$ so that our pseudo-data start close to the training data, but our results are insensitive to this particular choice of initialization. 

\paragraph{Computational Cost}
We provide a full discussion of the computational cost of RTDM in Appendix \ref{sec:computational_cost}. We note that the limiting step in the algorithm is computing the whitening matrix, $W$, which involves performing SVD on the moment estimate $\hat{M}_2$, resulting in a computational cost of $\mathcal{O}(D^2)$, and that the moments of the training data only need to be computed once and then cached. 

\section{Experiments}
\label{sec:experiments}
In the following, we provide a series of demonstrations on the properties and versatility of our approach on synthetic, semi-synthetic and real data under two generative models: Gaussian mixtures and LDA.  First, we demonstrate that given prior knowledge of the latent variables structure, our method can improve inference when data is sparse and ignore the prior knowledge when data is abundant.  Next, we demonstrate our approach on a range of different regularizers, including those targeting transfer, sparsity, orthogonality, and interpretability.  Finally, we explore the effect of different choices for the parameters $\lambda$ and $N_P$ and discuss guidelines for choosing them.

\paragraph{Basic Demonstration: Encouraging Prior Properties (with Gaussian mixtures).}
We first demonstrate our method for regularizing Gaussian mixture models. The generative process involves first selecting a mixture component $h_n$ out of $K$ choices with probability $w_k$ $(\sum_k w_k=1)$, and then sampling the data  $x_n$ from a multivariate normal $\mathcal{N}(a_{h_n}, \sigma^2)$. We generate the matrix $A$ of mixture components means by sampling from a normal distribution $a_{k,d} \sim \mathcal{N}(0, \sigma_m^2)$. The distribution from which $A$ is generated can be used as a prior when learning $A$, and therefore the regularization function is the log-likelihood of the learned $A_{T\cup P}$ under that prior:

\begin{align}
\label{eq:regularizer__gaussian_mixtures}
R(A_{T\cup P}) = \log{p(A_{T\cup P}|\sigma_m^2)} = -\frac{DK}{2}\log{(2 \pi \sigma_m^2)} - \frac{1}{2 \sigma_m^2}\sum_{d,k}a_{d,k}^2,
\end{align}

In Figure \ref{fig:synthetic__gaussian_mixtures} we demonstrate the RTDM on a synthetic example with $D=10$ dimensions and $K=4$ mixture components drawn from a prior with $\sigma_m^2=1$ and data variance of $\sigma^2 = 100$. We generate $N_T=200$ training data points and use the standard TDM to learn the mixture means. Because of the large variance in the data, the true means are difficult to learn as can be seen by the large errors of the learned means (cyan circles) compared to the true means (green diamonds). Applying the RTDM to regularize the data with $N_P=50$ pseudo-observations and $\lambda=0.5$, we see that we are able to use our knowledge of the prior to learn a more accurate estimate of the means (yellow dots). The right plot of Figure \ref{fig:synthetic__gaussian_mixtures} demonstrates the optimization where optimizing the loss function (Equations \ref{eq:cost_func__general} and \ref{eq:regularizer__gaussian_mixtures} - top) correlates with increasing the log-likelihood on a test data set of 1000 data points (bottom).

\begin{figure*}[ht]
\centering
\begin{subfigure}
\centering
\includegraphics[width=0.42\linewidth]{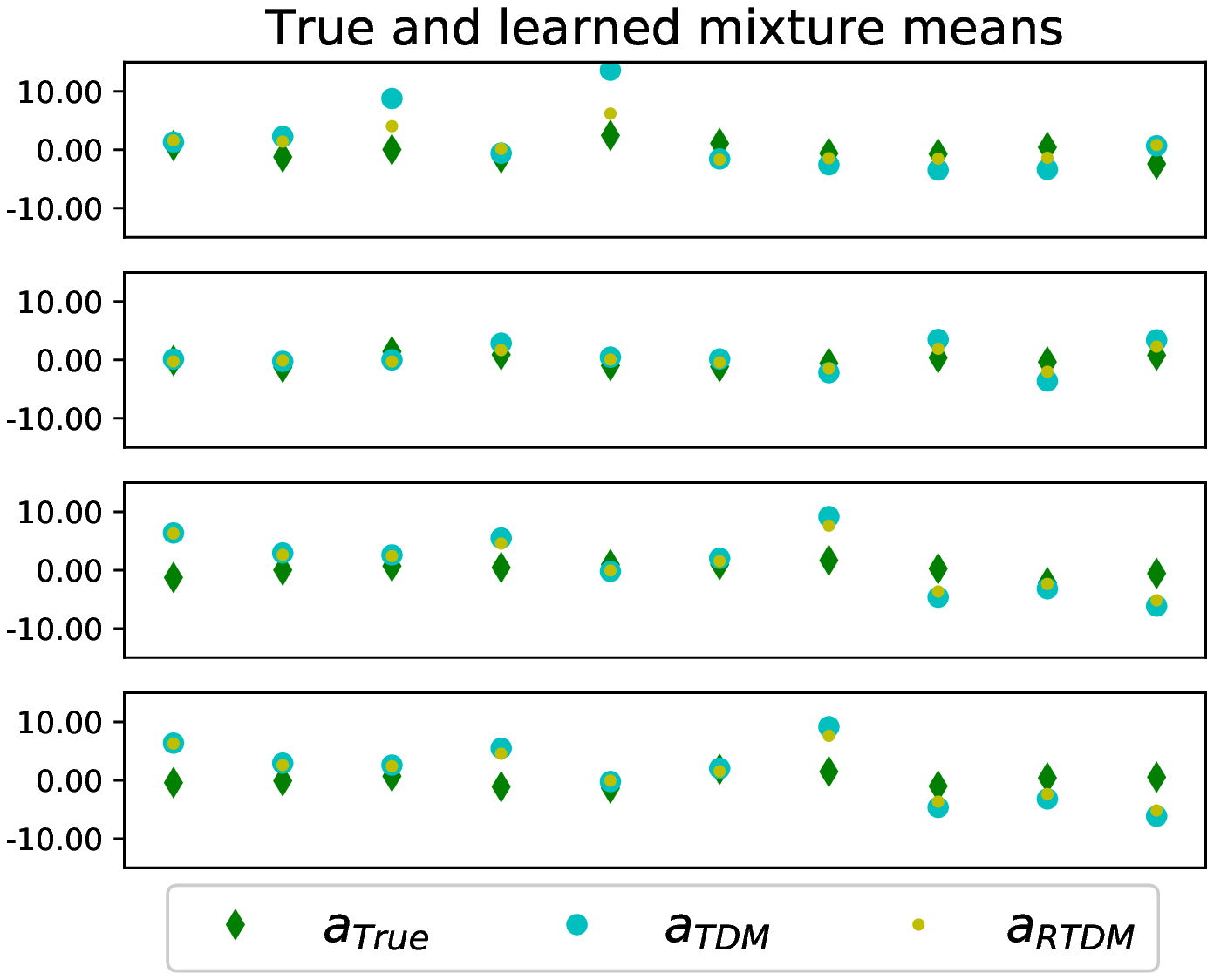}
\end{subfigure}
\begin{subfigure}
\centering
\includegraphics[width=0.45\linewidth]{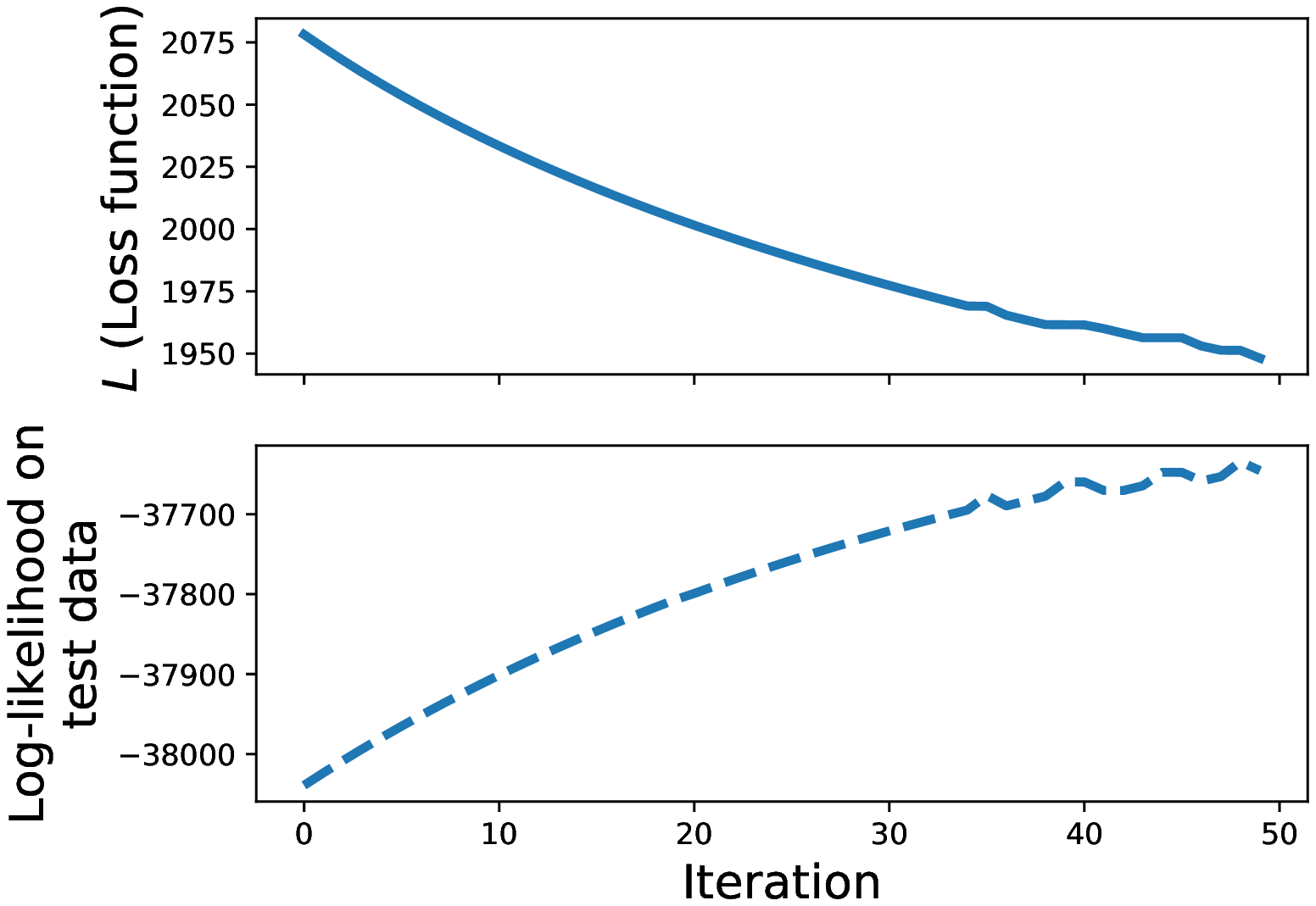}
\end{subfigure}
\caption{\textbf{Regularization with synthetic Gaussian Mixtures data.} RTDM uses prior knowledge of the generative distribution for the mixture means to regularize the TDM when training data is limited. The learned parameters with RTDM are closer to the true parameters (left), and optimization of the loss function $L$ leads to higher log-likelihood on a test set (right).}
\label{fig:synthetic__gaussian_mixtures}
\end{figure*}

\paragraph{Regularizing for Transfer Learning: Demonstration on LDA with semi-synthetic autism spectrum disorder patient data.}
We now demonstrate the RTDM on an LDA generative model, and show that regularization can be used for transfer learning --- i.e. use prior knowledge about the parameter structure of our data to improve inference when data is limited, and ignore it when the data is abundant. We use semi-synthetic autism spectrum disorder (ASD) data --- the data is a simulated dataset, but simulated from topics learned using real data, and we therefore expect these topics to include the sparsity and correlations which are representative of the true data~\citep{arora2013practical}. We use electronic health records of $D=64$ common diagnoses of children with autism~\citep{doshi2014comorbidity}. We use the real data to learn two topic matrices ($K=4$) representing common symptoms for two age groups, 6 to 7 and 8 to 9. We make the assumption that the symptoms of the two age groups share some similarities, and that we can transfer knowledge about one age group to better learn the characteristics of the other. In other words, we expect the topics learned for one age group to be an informative prior for inference on the other group.

To test our method over a range of $N_T$, we sample observations from the LDA model using the topic matrix of ages 6 to 7. We refer to this topic matrix as $A_{true}$, as it represents the true topics that we wish to learn. As a regularizer, we choose $R(A_{T\cup P}) = ||A_{T\cup P}-A_{prior}||_2$, where $A_{prior}$ is the topic matrix representing ages 8 to 9. In other words, we use RTDM to generate pseudo-data which encourages the learned topics from data on ages 6 to 7 to be as similar as possible to the topics learned from ages 8 to 9. In Figure \ref{fig:semisynthetic__ASD} (left) we demonstrate the results of our method for $N_T=100$ and $N_P=30$. We see that as we optimize our pseudo-data to decrease $R(A_{T\cup P})$ (solid blue line), the $L_2$ distance between the learned and true topics, $\varepsilon=||A_{T\cup P}-A_{true}||_2$, also decreases (dashed orange line). In Figure \ref{fig:semisynthetic__ASD} (right), we demonstrate the results for a similar experiment, only this time we choose $N_T=10000$. For this much larger dataset $\varepsilon$ is very small, as we have enough data to estimate $A_{true}$ properly. In this case, the $30$ pseudo points are overwhelmed by the true data, and our algorithm has little effect on the final learned topics (dashed orange line value changes very little).

\begin{figure*}[ht]
\centering
\begin{subfigure}
\centering
\includegraphics[width=0.45\linewidth]{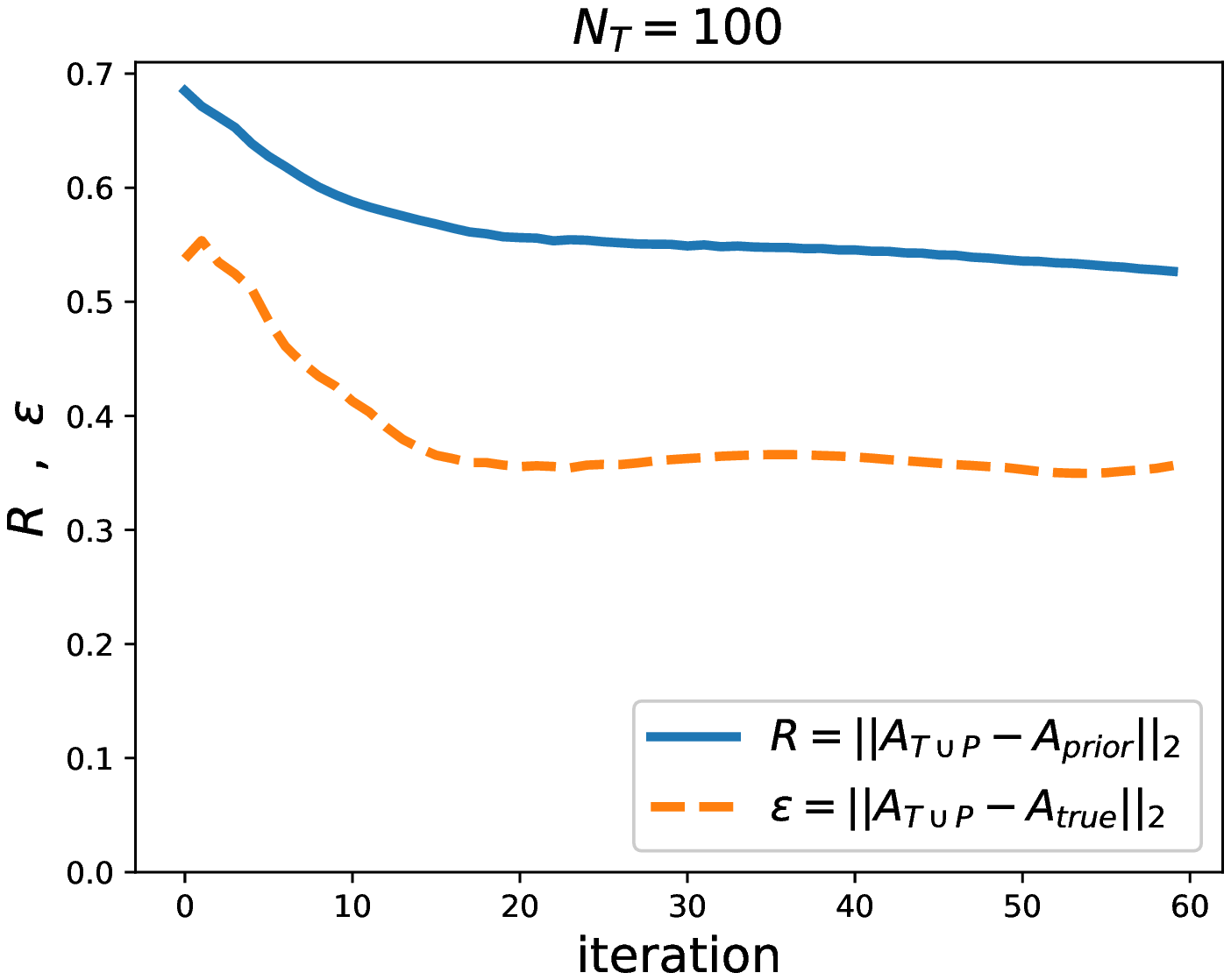}
\end{subfigure}
\begin{subfigure}
\centering
\includegraphics[width=0.45\linewidth]{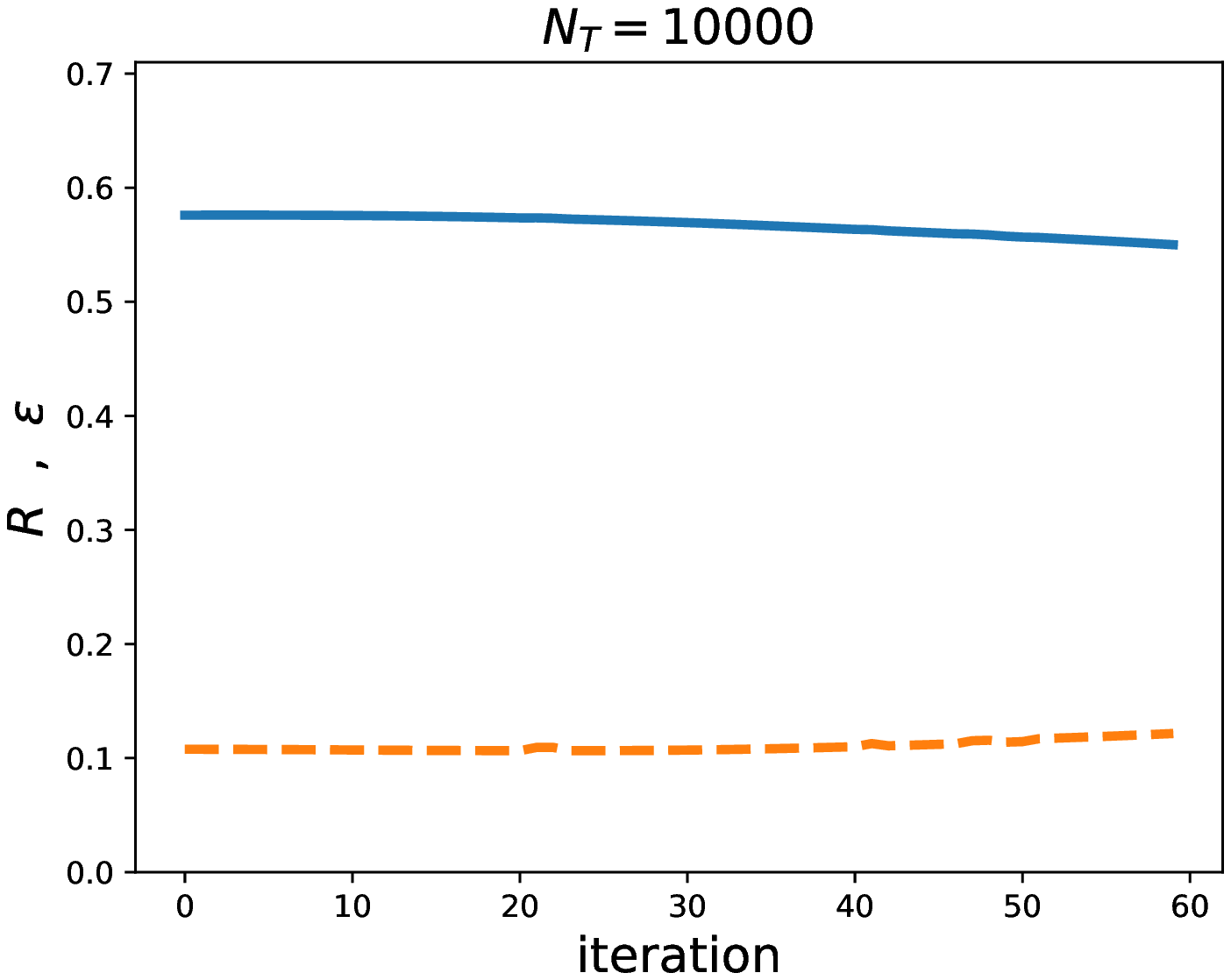}
\end{subfigure}
\caption{\textbf{Transfer learning with semi-synthetic ASD patient data.} $L_2$ distance from the learned topics to the prior topics (solid blue line) and to the true topic matrix (dashed orange line). For both examples $N_P=30$ and $\lambda=300$. For a small amount of training data ($N_T=100$ - left), RTDM allows for transfer learning which improves inference, and for a large amount of training data ($N_T=10000$ - right) our method has little affect on the relatively accurate topics estimation.}
\label{fig:semisynthetic__ASD}
\end{figure*}

\paragraph{Regularizing for Anti-Correlation: Demonstration with LDA; Exploration of the effect of model parameters.}
We now explore the effect of the choice of model parameters, $N_P$ and $\lambda$, on the performance of the RTDM in a setting where we use it on a synthetic LDA dataset to learn anti-correlated topics. Such topics structure leads to more diverse topics and can be useful, for example, in classification of text documents for which we wish each learned group of documents to have words which are unique to that specific group. Specifically, the regularizer we use is
\begin{align}
\label{eq:regularizer__anti_correlation}
R(A_{T\cup P}) = \sum_{i\neq j}a_i\cdot a_j,
\end{align}
where $a_i\cdot a_j$ denotes the dot product. This function is zero when all topics are orthogonal. We choose this particular regularizer for its simplicity, but note that other diversity promoting regularizers have been proposed in the literature~\citep{kwok2012priors}, and our method can easily be generalized to any other regularization function.

We generate synthetic data from an LDA model with $D=100$ dimensions and $K=4$ topics. In all experiments we use the same $X_T$ with $N_T=100$, but initialize a new set of $X_P$. In Figure \ref{fig:synthetic__topics_anti_correlation} we show the topics correlation of $A_{T\cup P}$ for different choices of $\lambda$ and $N_P$. Because the full loss function in Equation (\ref{eq:cost_func__general}) balances the regularizer and the log-likelihood of the pseudo-data, which is proportional to $N_P$, we plot the topics correlation vs. $\lambda/N_P$ rather than $\lambda$.

As either $\lambda$ or $N_P$ is increased, the pseudo-data pushes the learned topics towards higher diversity. However, for a given $N_P$, the influence of the pseudo-data is limited and at some point the topics correlation no longer decreases as $\lambda$ is increased. This implies that an alternative to tuning both parameters of our method could be to set $\lambda$ to a very high value, and adjust the weight we put on our regularizer by only tuning $N_P$.

\begin{figure}[ht]
\vskip 0.2in
\begin{center}
\centerline{\includegraphics[width=0.55\columnwidth]{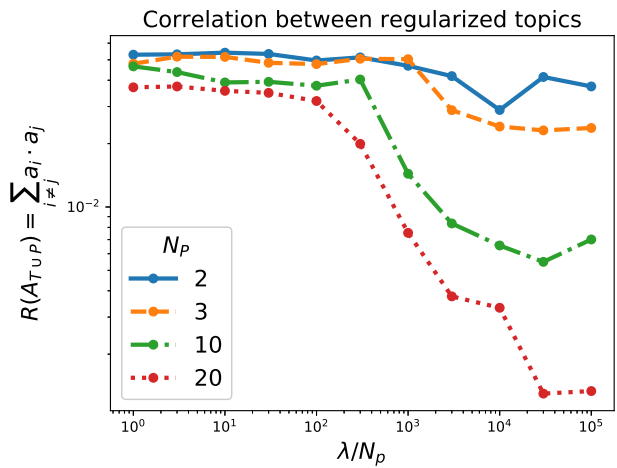}}
\caption{\textbf{Synthetic data - Regularizing for topics anti-correlation.} Topics correlation of $A_{T\cup P}$ for different choices of algorithm parameters ($D=100$, $N_T=100$). While the use of pseudo-data allows us to learn more diverse topics, the fact that at some point the correlation no longer decreases as $\lambda$ increases demonstrates that a given number of pseudo-observations is limited in the influence it can have on the learned topics, no matter how large the regularization weight is.}
\label{fig:synthetic__topics_anti_correlation}
\end{center}
\vskip -0.2in
\end{figure}

\paragraph{Regularizing for Interpretability via Hierarchy Matching: Demonstration with real data of Medical Subject Headings hierarchy.}
We demonstrate the performance of the RTDM with a complex regularization function used on real data. The National Library of Medicine (NLM) uses a hierarchically-organized terminology of medical subject headings (MeSH) for indexing medical articles\footnote{https://www.nlm.nih.gov/mesh/}. Every article is labeled using several headings, which are given an assignment on a hierarchical tree, in which the root represents a general topic, and headings become more specific farther down the tree. An example of three generations of headings in the tree is ``Adult [M01.060.116]'', ``Aged [M01.060.116.100]'' and ``Aged, 80 and over [M01.060.116.100.080]''. Formally, each three digit number in the full heading represents a node, and the periods separating them represent edges.

Topic modeling on subject headings of papers can help in identifying publication and research trends by finding headings which occur together frequently. Because articles are hand labeled, there could be significant inconsistency in labeling---e.g., a particular paper could be given each of the three headings presented in the example, depending on the particular person who labeled it~\citep{doshi2014graph}. A useful property of the topics which could help avoid missing information due to inconsistency in labeling is pushing for topics with headings which are close on the tree.

To encode tree-structured information into the model, we choose a regularizer that is much more complex than simply regularizing for sparsity or diversity, as we want to use our knowledge about the hierarchical indexing structure of the data. We use the following regularizer to achieve this property:
\begin{align}
\label{eq:regularizer__MeSH}
R(A_{T \cup P}) = - \sum_{k=1}^K(\sum_{i\neq j}a_{i,k}a_{j,k}O_{i,j}^{-1}),
\end{align}
where $O_{ij}$ is the distance on the tree between the $i^{th}$ and $j^{th}$ headings. In Appendix \ref{sec:mesh_regularizer_choice} we motivate the choice for this regularizer and explain why it promotes the desired property of learning topics with closely related headings.

We perform experiments on a labeled dataset of research articles on statins---a group of drugs for treating cardiovascular disease~\citep{cohen2006reducing}. Our training data consists of $N_T=500$ documents using the $D=300$ most common headings, and we learn a model with $K=3$ topics.

In Figure \ref{fig:MeSH} (left) we plot the value of the regularizer for the learned topics with different values of $N_P$ and $\lambda$, and observe that similarly to the last example, a given number of pseudo-observations is limited in the effect it can have on the topics, no matter how large $\lambda$ is.  In Figure \ref{fig:MeSH} (right) we plot the log-likelihood of a held out test set of 500 documents, using the learned topics. We see that in most cases, the interpretability of the topics comes at a relatively low cost in terms of prediction accuracy, which grows as $N_P$ is increased.

In Appendix \ref{sec:mesh_table} we illustrate the effect of regularization on the interpretability of the topics, by comparing the top eight headings of all three topics for the original and regularized topics with $N_P=100$ and $\lambda=1000$. We manually color headings on similar branches of the hierarchy tree for clarity. We see that the regularized topics include more headings from similar branches of the tree, and more noticeably, \emph{these headings have significantly higher weights than all other headings}.

\begin{figure*}[ht]
\centering
\begin{subfigure}
\centering
\includegraphics[width=0.45\linewidth]{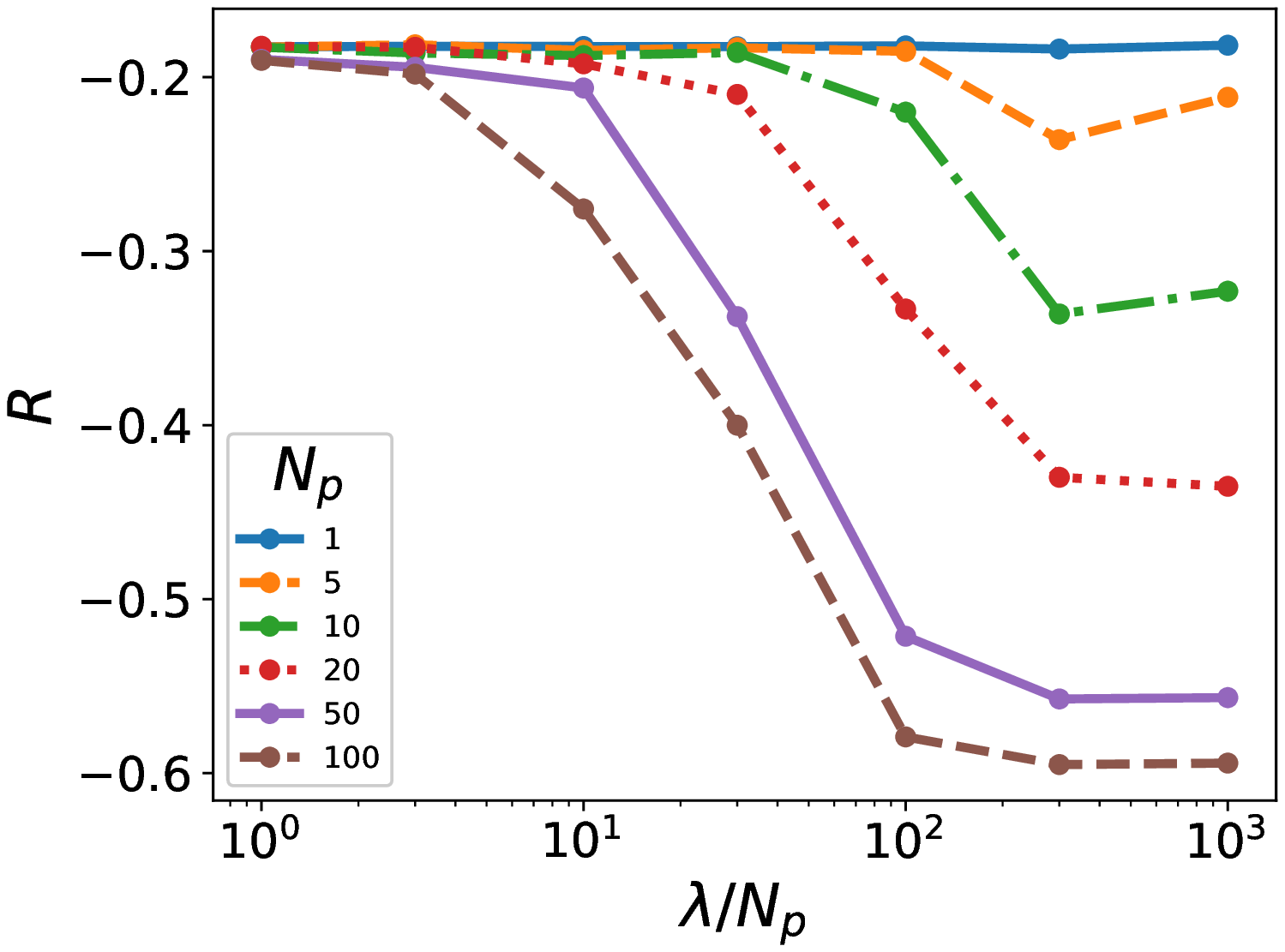}
\end{subfigure}
\begin{subfigure}
\centering
\includegraphics[width=0.48\linewidth]{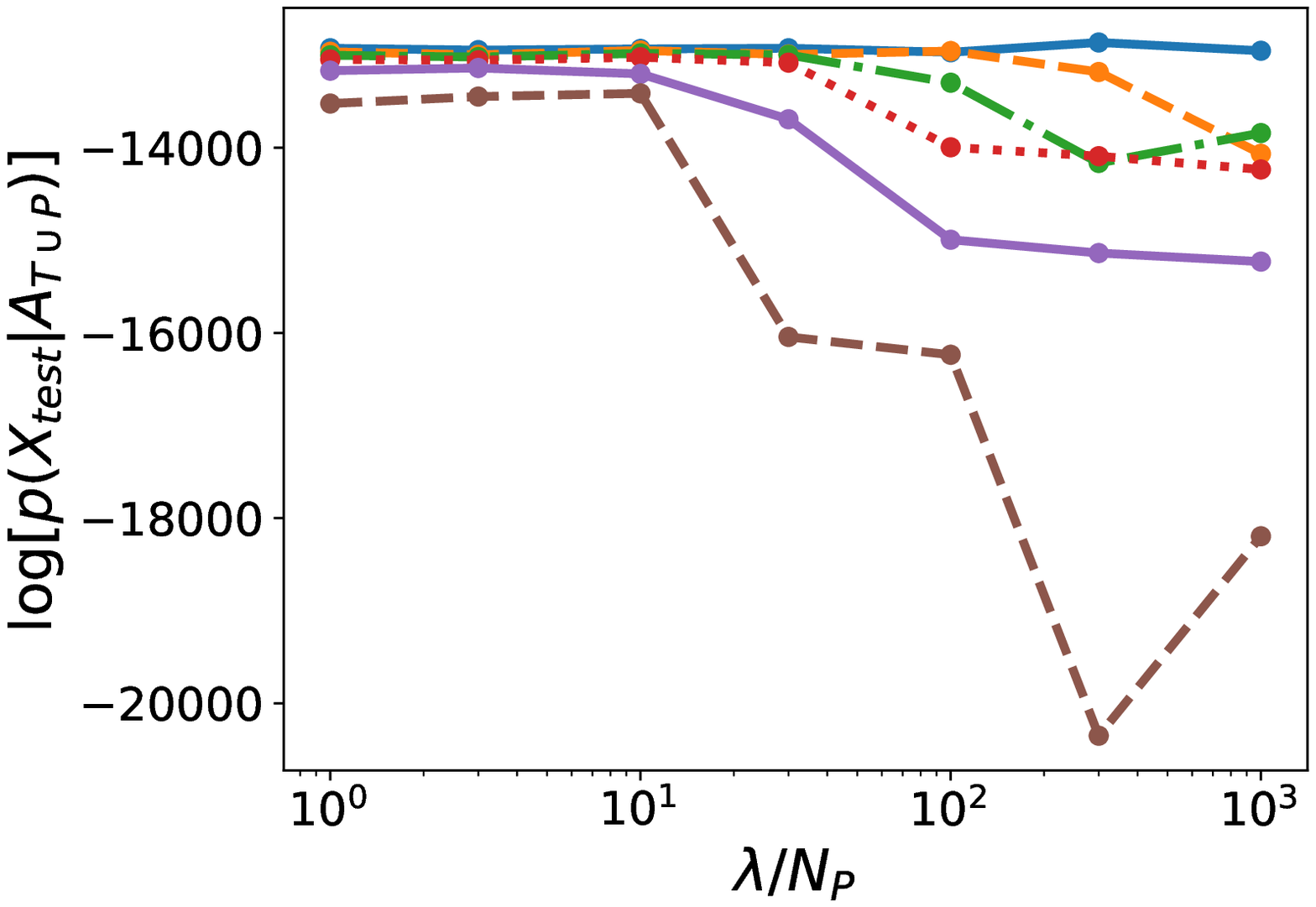}
\end{subfigure}
\caption{\textbf{Real data (MeSH) - Regularizing for interpretability.} $R(A_{T\cup P})$ value (left) and log-likelihood of held-out test set (right) for regularized topics with different values of $N_P$ and $\lambda/N_P$. The regularization function allows us to improve interpretability (lower $R(A_{T\cup P})$ - see Table \ref{table:MeSH_topics}) at the cost of predictive power (lower log-likelihood on test data). The examples with $N_P=20$ and $N_P=50$ demonstrate it is possible to obtain significant improvement in interpretability with a relatively small decrease in predictive power.}
\label{fig:MeSH}
\end{figure*}

\paragraph{Regularizing for Sparsity: Finding Structure in Noisy Data.}
We now consider an example in which the true structure is sparse but the data has been corrupted by additional noise (not part of the assumed generative model).  Specifically, we generate synthetic data according to the LDA model, but add to each element in $X$ noise sampled from a Poisson distribution, $\varepsilon\sim\text{Poisson}(\lambda_P)$. The topics themselves are sampled from a Dirichlet distribution ($a_k\sim\text{Dirichlet}(\alpha_A1_D)$; with $\alpha_A=0.1$). The noise causes the learning algorithm to attribute some probability mass to all dimensions, obscuring the sparse structure of the true topics. To address this problem, we use the log-liklihood of the Dirichlet probability distribution to regularize the topics,
\begin{align}
R(A_{T \cup P})=-\log p(A_{T \cup P}|\alpha_A).
\end{align}

In Figure \ref{fig:synthetic__sparsity_for_noise} (left) we demonstrate that as the optimization learns topics that minimize $R(A_{T\cup P})$, the $L_2$ norm between the true and inferred topics also decreases. The way in which the algorithm improves inference can be understood by comparing the true topics with the topics learned before and after optimization. The right plot in Figure \ref{fig:synthetic__sparsity_for_noise} represents one of the $K=4$ topics, where the x-axis represents the index out of the $D=100$ dimensions and the y-axis the value of that dimension on the simplex. The true---very sparse---topics are represented by green diamonds. Because of the noise, the initial topics learned by the standard tensor decomposition algorithm (cyan circles) are not sparse. The algorithm finds the dimensions which have a significant probability mass, but are just barely above the noise. The Dirichlet prior over the topics prefers sparser topics, which it generates by amplifying the value of the dimensions above noise level, resulting in learned topics which are much closer to the true topics (yellow squares).

We note that unlike all other examples in this paper, where our method's convergence to the standard tensor decomposition method is desirable when data is abundant, in this case abundance of data will prevent our method from sparsifying the topics, as the noise-free model is misspecified. We may increase $N_P$ to allow our algorithm to sparsify topics learned from large noisy datasets---corresponding to using very sharply peaked priors.

\begin{figure*}[ht]
\centering
\begin{subfigure}
\centering
\includegraphics[width=0.45\linewidth]{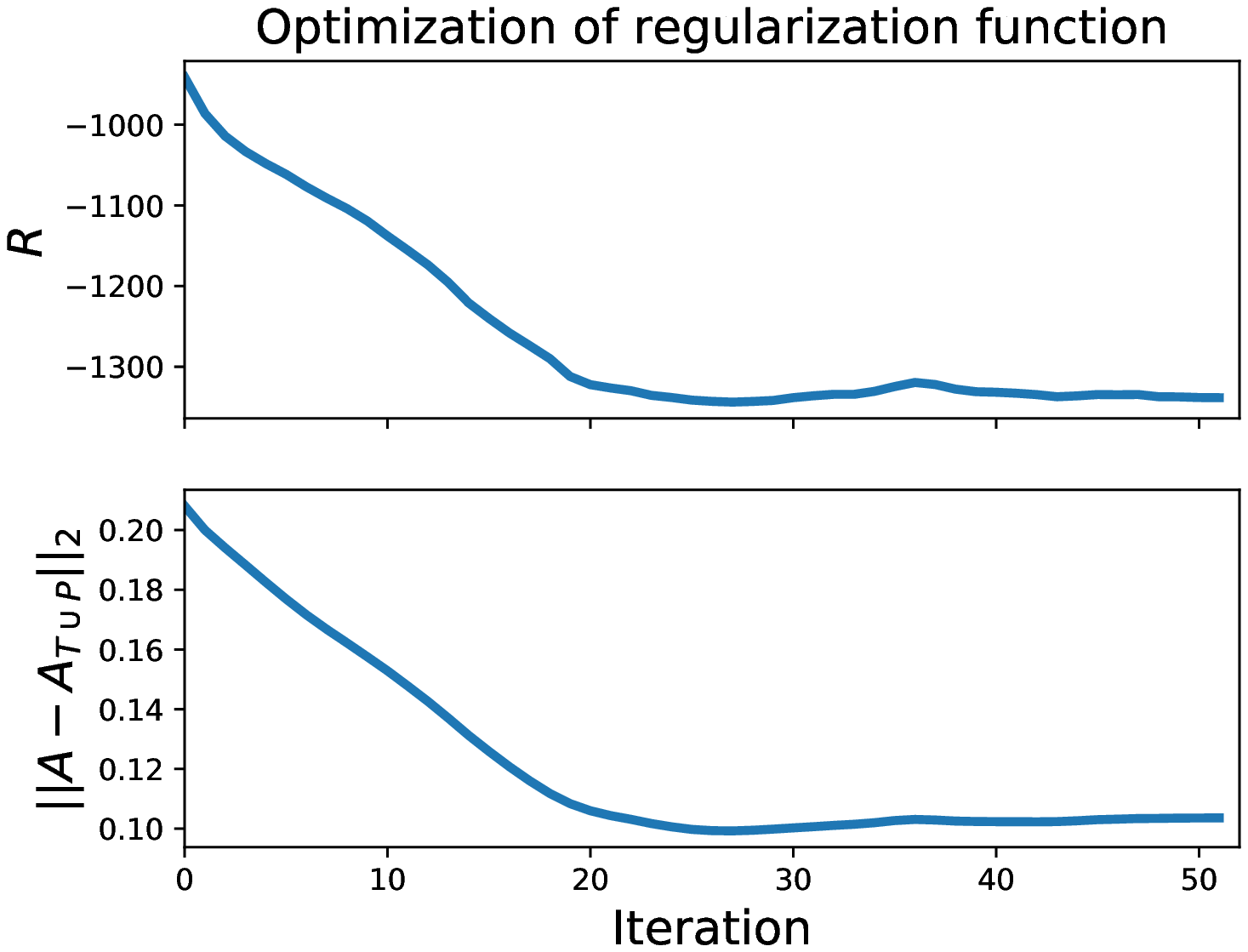}
\end{subfigure}
\begin{subfigure}
\centering
\includegraphics[width=0.45\linewidth]{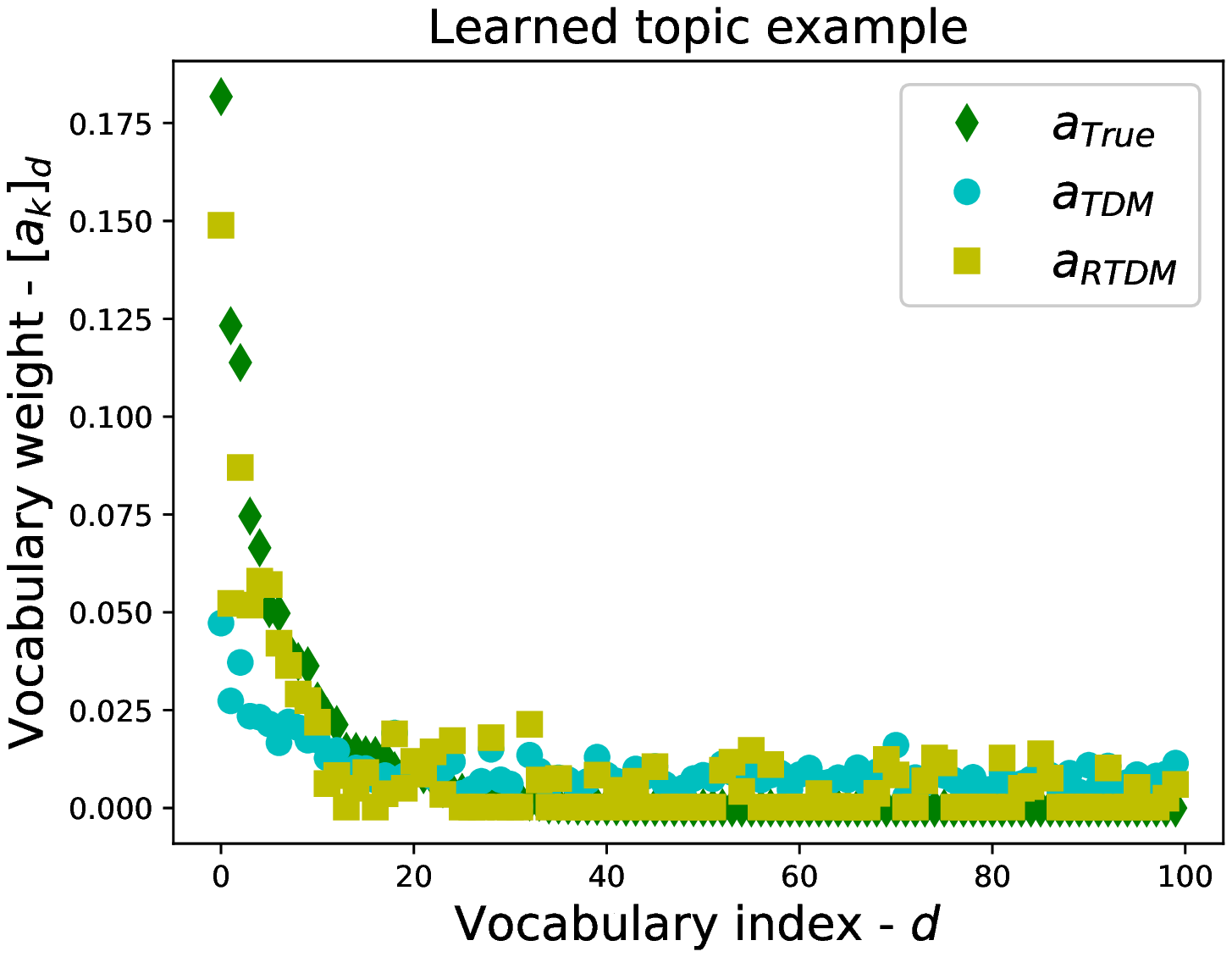}
\end{subfigure}
\caption{\textbf{Synthetic data with noise - Regularizing for sparsity.} (left) Optimization of $R(A_{T\cup P})$ (top) and corresponding improvement in reconstruction error (bottom). (right) An example of the feature weights for one of the learned topics before and after regularization, compared with the true topics. We see that regularizing for sparsity improves the learning quality by amplifying the weight on feature with signal above the noise level.}
\label{fig:synthetic__sparsity_for_noise}
\end{figure*}

\section{Discussion}
\label{sec:discussion}

The proposed tensor decomposition regularization algorithm requires two parameters---$N_P$ and $\lambda$. We demonstrated throughout this paper that for a given $N_P$, at some point increasing $\lambda$ no longer influences the final regularized topics. The intuition behind this saturation is that there is a limit to how big of an effect a small fraction of the data can have on the learned topics. In practice, this means we can choose $\lambda$ to be in the saturated regime (high $\lambda$), and only tune $N_P$ to control the strength of our regularizer, reducing the number of parameter choices required.

More broadly, we provide a general approach for regularizing tensor decomposition methods via pseudo-data.  The designer has full control over the properties they wish to induce via the regularization, and in this sense, our approach is different from, for example, simply bootstrapping the data (which might provide robustness, but not induce desired properties) or bespoke methods based on specific properties (e.g. anchor words).  
%
%
%
A strength of this method lies in its ability to use regularization and prior knowledge to improve learning when data is limited, and ignore our prior beliefs when data is abundant, much like the effect of priors in a Bayesian setting.  Future work could look at alternate optimization techniques for the pseudo-data, including ways to encourage the collection of pseudo-data to have similar properties to the training data. Another open direction is drawing a tighter connection between our method and Bayesian inference. 

\bibliographystyle{plainnat}
\bibliography{bib}

\clearpage
\begin{appendices}

\section{Empirical moments expressions for LDA and Gaussian mixtures}
\label{sec:empirical_moments}

For convenience, we present here the expressions for the empirical estimates of the low order moments of the data required to perform TDM on the models for which we present experimental results --- LDA and Gaussian mixture models. For a full description of TDM we refer the reader to \citet{anandkumar2014tensor}.

As a reminder from Section \ref{sec:background} in the main text, tensor decomposition methods learn the parameters of a latent variable model given as a matrix $A\in\mathbb{R}^{D\times K}$, where $D$ is the dimensionality of the data and $K$ the number of latent variables. TDM works by leveraging the relationship between the empirical moments of the data and the latent parameters of the model. Specifically, $A$ is learned by matching the theoretical moments of the model,
\begin{align}
M_2 & = \sum_{k=1}^K \beta_k a_k a_k^T, \label{eq:M2_decomposition}\\
M_3 & = \sum_{k=1}^K \gamma_k a_k \otimes a_k \otimes a_k, \label{eq:M3_decomposition}
\end{align}
with their empirical estimates, which can be computed from data and are provided below. Here $a_k$ is the $k^{th}$ column of $A$, and $\beta_k$ and $\gamma_k$ are constants that come from the model parameters (depend on each generative model).

\subsection{Empirical moments for LDA} The generative model for LDA is
\begin{align}
b_n&\sim\text{Dirichlet}(\alpha_B1_K) \\
z_{in}|b_n&\sim\text{Discrete}(b_n) \\
w_{in}|z_{in},A&\sim\text{Discrete}(a_{z_{in}}), \label{eq:LDA_model__feature_value}
\end{align}
where $b_n$ is the distribution over topics of sample $n$. $w_{in}$ is the feature assignment of the $i^{th}$ feature in sample $n$, where $w_{in}\in1,..,D$. $z_{in}$ is the topic assignment of the $i^{th}$ feature in sample $n$. We note that a common formulation of LDA also includes a prior on the topics, $a_k$. We omit this prior here as it is not used in TDM.

Given a dataset, we can learn the topics structure by computing the following empirical moments of the data:
\begin{align}
\label{eq:LDA_empirical_M1}
\hat{M}_1 &= \mathbb{E}[e_{w_1}] \\
\label{eq:LDA_empirical_M2}
\hat{M}_2 &= \mathbb{E}[e_{w_1}\otimes e_{w_2}] - \frac{K\alpha_B}{K\alpha_B+1} \hat{M}_1 \otimes \hat{M}_1 \\ 
\label{eq:LDA_empirical_M3}
\hat{M}_3 &= \mathbb{E}[e_{w_1}\otimes e_{w_2}\otimes e_{w_3}]  \nonumber \\
&-\frac{K\alpha_B}{K\alpha_B+2}(\mathbb{E}[e_{w_1}\otimes e_{w_2}\otimes \hat{M}_1] \nonumber \\
&+\mathbb{E}[e_{w_1}\otimes \hat{M}_1 \otimes e_{w_2}] + \mathbb{E}[\hat{M}_1\otimes e_{w_1}\otimes e_{w_2}]) \nonumber \\
&+\frac{2(K\alpha_B)^2}{(K\alpha_B+2)(K\alpha_B+1)}\hat{M}_1 \otimes \hat{M}_1 \otimes \hat{M}_1,
\end{align}
where $e_{w_i}\in\{0,1\}^D$ is the vector whose only non-zero element corresponds to the feature assignment, $w_i$. The topics are found by decomposing the empirical moments in Equations \ref{eq:LDA_empirical_M2} and \ref{eq:LDA_empirical_M3} and matching them with the moments in Equations \ref{eq:M2_decomposition} and \ref{eq:M3_decomposition}, with $\beta_k = \beta = \alpha_B/(K\alpha_B+1)K\alpha_B$ and $\gamma_k = \gamma = 2\alpha_B/(K\alpha_B+2)(K\alpha_B+1)K\alpha_B$.

\subsection{Empirical moments for Gaussian mixtures}
The generative model for (spherical) Gaussian mixture models is
\begin{align*}
  h_n &\sim \text{Multinomial}(1,w), \\
  x_n|h_n,A &\sim \mathcal{N}(a_{h_n},\sigma^2).
\end{align*}

where $a_{h_n}$ is the $(h_n)^{th}$ column $A\in \mathbb{R}^{D\times K}$ representing the mixture means, and $w\in \mathbb{R}^K$ represents the probability of data points to be drawn from each topic $(\sum_{k=1}^K w_k = 1)$.

The empirical moments of the data are given by \citep{hsu2013learning}

\begin{align}
\label{eq:tensor_estimates_Gauss_2}
\hat{M}_2 &= \mathbb{E}[x \otimes x] - \sigma^2I \\
\label{eq:tensor_estimates_Gauss_3}
\hat{M}_3 &= \mathbb{E}[x \otimes x \otimes x] - \sigma^2 \sum_{i=1}^D(\mathbb{E}[x]\otimes e_i \otimes e_i \notag\\
  &+ e_i \otimes \mathbb{E}[x]\otimes e_i + e_i \otimes e_i \otimes \mathbb{E}[x]),
\end{align}
where $\sigma^2$ is estimated by the smallest eigenvalue of the covariance matrix, $ \mathbb{E}[x \otimes x] - \mathbb{E}[x] \otimes \mathbb{E}[x]$. The topics are found by decomposing the empirical moments in Equations \ref{eq:tensor_estimates_Gauss_2} and \ref{eq:tensor_estimates_Gauss_3} and matching them with the moments in Equations \ref{eq:M2_decomposition} and \ref{eq:M3_decomposition}, with $\beta_k = \gamma_k = w_k$.

\section{Convergence of RTDM to true model parameters}
\label{sec:proof_of_consistency}

\begin{theorem}
Fix a likelihood model, $p(X|A)$, and a regularizer, $R(A)$, that is bounded below by a constant $B_R$. For any fixed nonnegative $\lambda$ and $N_P$, as $N_T \rightarrow \infty$, minimizing the cost function, $L$, in Equation \ref{eq:cost_func__general} with respect to $X_P$ results in $A_{T \cup P} \rightarrow A_T$.
\end{theorem}

\begin{proof}

The model parameters matrix, $A_{T \cup P}$, is the result of applying a standard tensor decomposition method with the estimated $i^{th}$ order  moments $\hat{M}_{i, T \cup P}$, for $(i\in\{2,3\})$, computed using the augmented data, $X_{T \cup P}$, containing the training and the pseudo-data. Since $\hat{M}_{i, T \cup P}$ is an estimated expectation of some function of the data (depending on the generative model being learned), we can write $\hat{M}_{i, T \cup P}$ as a weighted sum of the $i$-th moment estimated using the training data, $\hat{M}_{i,T}$, and the that estimated using the pseudo data, $\hat{M}_{i,T}$:

\begin{align}
\hat{M}_{i, T \cup P} &= \frac{N_T}{N_T+N_P}\hat{M}_{i,T} + \frac{N_P}{N_T+N_P} \hat{M}_{i,P} \\ \nonumber
&= \hat{M}_{i,T} + \frac{N_P}{N_T+N_P}(\hat{M}_{i,P} - \hat{M}_{i,T}),
\end{align}

Therefore, if $\|\hat{M}_{i,P} - \hat{M}_{i,T}\|_2$ can be bounded by some constant $B_{M_i}$ for all $N_T$ and all training data $X_T$ of size $N_T$, then we have that the training data dominates the pseudo data in the moment estimation as as $N_T$ increases. More formally, we can conclude that

\begin{align}
\lim_{N_T \rightarrow \infty} \hat{M}_{i, T \cup P} = 
\lim_{N_T \rightarrow \infty} \left[ \hat{M}_{i, T} + \mathcal{O} \left( \frac{N_P}{N_T} \right) \right] = 
\hat{M}_{i, T}.
\end{align}

It then follows that as the training data size $N_T$ increases, $A_{T \cup P}$ converges to the parameter matrix computed by a standard TDM using the training data alone, i.e. $A_{T \cup P}$ converges to $A_T$.

It remains to show that $\|\hat{M}_{i,P} - \hat{M}_{i,T}\|_2 \leq B_{M_i}$ for all $N_T$ and $|X_T| = N_T$. We first note that $\hat{M}_{i,T}$ converge to the true (finite) moment $M_i$ of the model parameters, that is, for any $\epsilon > 0$ there is some integer $N_\epsilon$ such that $\|\hat{M}_{i,T} - M_i\|_2 \leq \epsilon$ for all $|X_T| > N_\epsilon$. On the other hand, since the regularizer term, $R(A_{T \cup P})$, in Equation \ref{eq:cost_func__general} is bounded below, for any training data $X_T$, there must exist some $\delta_{X_T}$-ball, $\mathcal{B}(\delta_{X_T})$, centered at the the mean of the theoretical distribution of $X_T$ such that $log(X_P|A(X_T)) > B_R$ for all $X_P$ sampled outside $\mathcal{B}(\delta_{X_T})$. Thus, the pseudo data $X_P^*$ that minimizes the cost function $L$ in Equation \ref{eq:cost_func__general} must lie within $\mathcal{B}(\delta_{X_T})$. As $N_T$ become sufficiently large, we can assume that $\delta_{X_T}$ is constant for all training data set $|X_T| = N_T$, and, hence, that $\|\hat{M}_{i,p}\|_2$ is bounded, say by some constant $\gamma$. Thus, for $N_T > N_\epsilon$, we have the following for any $|X_T| = N_T$: 
\begin{align}
\|\hat{M}_{i,P} - \hat{M}_{i,T}\|_2 &\leq \|\hat{M}_{i,T} - M_i\|_2 + \|\hat{M}_{i,P} - M_i\|_2\\
&\leq \|\hat{M}_{i,T} - M_i\|_2 + \|\hat{M}_{i,P}\|_2 + \|M_i\|_2 \\
&\leq \epsilon + \gamma + \|M_i\|_2 \overset{\text{def}}{=}B_{M_i}.
\end{align}
The above shows that $\|\hat{M}_{i,P} - \hat{M}_{i,T}\|$ is bounded by a fixed constant $B_{M_i}$ and thus completes the proof.
\end{proof}

The assumption of boundedness of the regularizing function $R(A_{T \cup P})$ can be relaxed: for $X_P$ sampled from $\delta$-ball centered at the origin, we assume that the regularization term grows more slowly than the log-likelihood term as $\delta \to \infty$. However, most reasonable regularizers have some notion of "optimal" model parameters, in the sense that there is a particular minimal value for the models parameters which would have been chosen if the training data would have been completely ignored, and therefore the lower boundedness of the regularizer is a very mild assumption from a practical standpoint.

Because TDM are provably converge to the true model parameters when there is no model mismatch \citep{anandkumar2014tensor, anandkumar2012spectral, hsu2013learning}, a corollary of Theorem \ref{theorem:convergence}, is that the RTDM algorithm converges to the true model parameters, $A$. Furthermore, the convergence of the TDM is shown through perturbation analysis of the TDM to perturbation of the estimated tensors, $\hat{M}_i$. If these analyses are robust to small perturbations, the additional perturbation of $\hat{M}_i$ due to the pseudo-data, which is shown in the proof above to be of order $\mathcal{O} \left( \frac{N_P}{N_T} \right)$ for large $N_T$ can be plugged-in to these analyses to obtain the convergence rate of the RTDM to the true model parameters. In particular, for LDA \citep{anandkumar2012spectral} and Gaussian mixtures \citep{hsu2013learning}, the parameters reconstruction error bound was shown to be linear in the moments perturbation, and therefore for these models the RTDM converges to the true topics at a rate of $\mathcal{O} \left( \frac{N_P}{N_T} \right)$.

\section{Computational Cost}
\label{sec:computational_cost}

The computational complexity of the tensor decomposition algorithm as it appears in \citet{anandkumar2014tensor} is $\mathcal{O}(D^3)$, where the limiting step is in computing $\hat{M}_3\in\mathbb{R}^{D\times D\times D}$ and whitening it to the $\mathbb{R}^{K\times K\times K}$ tensor, $\hat{M}_{3,w}$. \citet{zou2013contrastive} demonstrated that for sparse data, the tensor decomposition can be performed in $\mathcal{O}(DK+nnz(X))$ where $nnz(X)$ is the number of non-zero elements in $X$. Because our algorithm is based on differentiating the results of the tensor decomposition algorithm with respect to its input, if we wish for our algorithm to be flexible enough to impose any regularizer, $X_P$ will generally not be sparse, and we cannot use the method introduced in \citet{zou2013contrastive}. Instead, we first compute $\hat{M}_2$, and use the whitening matrix, $W$ to whiten the data. We then compute $\hat{M}_{3,w}$ directly from the whitened data $X_w$, and never explicitly compute $\hat{M}_3$. This makes the limiting step in the algorithm the SVD computation of $\hat{M}_2$, and the computational complexity of the algorithm is $\mathcal{O}(D^2)$. We note that even though for high dimensional data with a large number of samples, $N_T$, the cost of constructing $\hat{M}_2$ is of order $\mathcal{O}(N_T D^2)$, this step must be preformed only once. For each subsequent optimization step we only need to recompute the pseudo-data contribution to $\hat{M}_2$ which has computational cost of $\mathcal{O}(N_P D^2)$, and in such cases we expect $N_P$ to be much smaller than $N_T$.

\section{Choice of regularizer for MeSH dataset}
\label{sec:mesh_regularizer_choice}

For the task of regularizing the model we learn for the MeSH data we choose a regularizer which is much more complex than simply regularizing for sparsity or diversity, as we want to use our knowledge about the hierarchical indexing structure of the data. We use the following regularizer to achieve this property---
\begin{align}
R(A) = - \sum_{k=1}^K(\sum_{i\neq j}A_{ik}A_{jk}O_{ij}^{-1}).
\end{align}
where $O_{ij}$ is the distance on the tree between the $i^{th}$ and $j^{th}$ headings. Minimizing this regularizer rewards topics with several headings which are close to each other on the tree ($-O_{ij}^{-1}$ is more negative), while simultaneously pushing for topics with at least more than one highly weighted heading. The latter property is important in understanding why we chose this regularizer over the more straightforward regularizer of $\sum_{k=1}^K\sum_{i\neq j}A_{ik}A_{jk}O_{ij}$---the total distance of headings on the tree weighted by their probabilities. Minimizing the weighted total distance between headings can be artificially reduced by making sparser topics and is trivially zero for a topic which has all of its probability on one heading, whereas the regularizer we use prefers the probability mass to be spread across as many headings as possible.\footnote{To see this, consider the case where all elements in the vector $a$ are zero except for $D_*$ entries with the value $1/D_*$. Then $-\sum_{i\neq j}a_ia_j= -{{D_*}\choose{2}} \frac{1}{D_*^2}=-\frac{1}{2}(1-\frac{1}{D_*})\sim\frac{1}{D_*}$, implying that the $-\sum_{i\neq j}a_ia_j$ term in the regularizer scales in inverse proportion to the number of non-zero elements.} While sparsity itself is often considered a desirable property for interpretability, in practice we observed that minimizing the weighted distance between headings tended to produce topics which were too sparse, leading us to choose the regularizer we present in Equation \ref{eq:regularizer__MeSH}.

To make computation easier and more efficient, $R$ can be written in matrix notation as
\begin{align}
\label{eq:regularizer__MeSH__matrix_notation}
R(A) = - \frac{1}{2}\textrm{tr}(A^TO^*A-A^TA),
\end{align}
where $O^*_{ij}=O^{-1}_{ij}$ for $i\neq j$ and $O^*_{ij}=1$ for $i=j$.

\section{Learned topics for MeSH data}
\label{sec:mesh_table}

\begin{table*}[!ht]
\vskip 0.15in
\begin{center}
\begin{small}
\begin{sc}
\begin{tabular}{llll}
\toprule
Weight & Initial topics & Weight & Final topics \\
\midrule
\color{blue} 0.0402 & \color{blue} Female & \color{blue} 0.4451 & \color{blue} Male \\
\color{JungleGreen} 0.0393 & \color{JungleGreen} Risk Factors & \color{blue} 0.4281  & \color{blue} Female \\
\color{blue} 0.0367 & \color{blue} Male & \color{magenta} 0.0349 & \color{magenta} Middle Aged \\
0.0361 & Humans & 0.0175 & Adult \\
0.0359 & Aged & 0.0166 & Aged \\
0.0330 & Middle Aged & \color{JungleGreen} 0.0045 & \color{JungleGreen} Double-Blind Method \\
0.0277 & Cardiovascular Diseases & \color{magenta} 0.0028 & \color{magenta} Aged, 80 and over \\
\color{JungleGreen} 0.0272 & \color{JungleGreen} Clinical Trials as Topic & \color{JungleGreen} 0.0025 & \color{JungleGreen} Follow-Up Studies \\
\hline
0.0717 & Clinical Trials as Topic & \color{red} 0.4056 & \color{red} \color{red} Anticholesteremic Agents \\
0.0597 & Animals & \color{red} 0.3108 & \color{red} Hydroxymethylglutaryl-CoA ... \\
\color{red} 0.0526 & \color{red} Hydroxymethylglutaryl-CoA ... & \color{red} 0.2532 & \color{red} Hypolipidemic Agents \\
0.0496 & Humans & 0.0002 & Fenofibrate \\
0.0376 & Coronary Disease & 0.0002 & Fibrinolysis \\
\color{red} 0.0368 & \color{red} Anticholesteremic Agents & 0.0001 & Platelet Membrane Glycoproteins \\
0.0319 & Hypercholesterolemia & 0.0001 & Erythrocytes \\
\color{red} 0.0284 & \color{red} Hypolipidemic Agents & 0.0001 & Mutagenicity Tests \\
\hline
\color{blue} 0.0559 & \color{blue} Male & \color{blue} 0.0483 & \color{blue} Male \\
\color{blue} 0.0539 & \color{blue} Female & \color{blue} 0.0477 & \color{blue} Female \\
0.0537 & Middle Aged & \color{red} 0.0476 & \color{red} Anticholesteremic Agents \\
0.0485 & Anticholesteremic Agents & \color{red} 0.0445 & \color{red} Hydroxymethylglutaryl-CoA ... \\
0.0432 & Hypercholesterolemia & 0.0333 & Humans \\
0.0384 & Adult & 0.0319 & Middle Aged \\
0.0352 & Aged & \color{red} 0.0288 & \color{red} Hypolipidemic Agents \\
0.0338 & Humans & 0.0281 & Clinical Trials as Topic \\
\bottomrule
\end{tabular}
\end{sc}
\end{small}
\end{center}
\caption{\textbf{Top headings for learned topics} --- Headings which are close to each other on the hierarchy tree are color coded with the same color (black signifies no neighbors within the topic)}
\label{table:MeSH_topics}
\vskip -0.1in
\end{table*}


\end{appendices}

\end{document}